\documentclass[journal]{IEEEtran}
\usepackage{amsmath}
\usepackage{graphicx}
\usepackage{autobreak}
\usepackage[hidelinks]{hyperref}
\usepackage{amsthm}
\usepackage{multirow}
\usepackage{booktabs}
\usepackage{algorithm}
\usepackage{algorithmic}
\usepackage{makecell}
\usepackage{multirow}
\usepackage{color}
\usepackage{amsfonts}
\usepackage{ulem}
\usepackage{bm}

\newtheorem{proposition}{Proposition}

\newtheorem{theorem}{Theorem}

\allowdisplaybreaks
\ifCLASSOPTIONcompsoc
  \usepackage[nocompress]{cite}
\else
  \usepackage{cite}
\fi
\ifCLASSINFOpdf
\else
\fi
\begin{document}
\setlength{\textfloatsep}{-1pt}
\title{REST: Debiased Social Recommendation via Reconstructing Exposure Strategies}


%
%
%
%

\author{Ruichu Cai$^\star$,~\IEEEmembership{Member,~IEEE,}
	Fengzhu Wu, Zijian Li, Jie Qiao, Wei Chen, Yuexing Hao, Hao Gu
	\IEEEcompsocitemizethanks{
		\IEEEcompsocthanksitem Ruichu Cai is with the School of Computer Science, Guangdong University of Technology, Guangzhou, China, 510006 and Peng Cheng Laboratory, Shenzhen, China, 518066.
		E-mail: cairuichu@gmail.com
		\IEEEcompsocthanksitem Fengzhu Wu is with the School of Computer, Guangdong University of Technology, Guangzhou China, 510006.
		E-mail: fzwu97@gmail.com
		\IEEEcompsocthanksitem Zijian Li is with the School of Computer, Guangdong University of Technology, Guangzhou China, 510006.
		E-mail: leizigin@gmail.com
		\IEEEcompsocthanksitem Jie Qiao is with the School of Computer Science, Guangdong University of Technology, Guangzhou, China, 510006. E-mail: qiaojie.chn@qq.com
		\IEEEcompsocthanksitem Wei Chen is with the School of Computer Science, Guangdong University of Technology, Guangzhou, China, 510006 and Peng Cheng Laboratory, Shenzhen, China, 518066. E-mail: chenweidelight@gmail.com
		\IEEEcompsocthanksitem Yuexing Hao is with Tufts University. E-mail: yhao02@tufts.edu
		\IEEEcompsocthanksitem Hao Gu is with Tencent Technology (SZ) Co., Ltd. E-mail: nickgu@tencent.com
		}
	\thanks{Manuscript received XX; revised XX; accepted XX. Date of publication XX XX, 2019; date of current version XX XX, 2019. This research was supported in part by National Key R\&D Program of China (2021ZD0111501), National Science Fund for Excellent Young Scholars (6212200101) and Natural Science Foundation of China (61876043, 61976052). Wei Chen was supported by China Postdoctoral Science Foundation (2021M690734). (*Ruichu Cai is the Corresponding author.)
	}
}

%
%

\markboth{IEEE Transactions on Neural Networks and Learning Systems,~submitted}%
{Ruichu Cai \MakeLowercase{\textit{et al.}}: REST: Debiased Social Recommendation via Reconstructing Exposure Strategies}
%



\IEEEtitleabstractindextext{%
\begin{abstract}
The recommendation system, relying on historical observational data to model the complex relationships among the users and items, has achieved great success in real-world applications. Selection bias is one of the most important issues of the existing observational data based approaches, which is actually caused by \textcolor{black}{multiple types of} unobserved exposure strategies (e.g. promotions and \textcolor{black}{holiday effects}). Though various methods
have been proposed to address this problem, they are mainly relying on the implicit debiasing techniques but not explicitly modeling the unobserved exposure strategies.
By explicitly \textbf{R}econstructing \textbf{E}xposure \textbf{ST}rategies (REST in short), we formalize the recommendation problem as the counterfactual reasoning and propose the debiased social recommendation method. In \textbf{REST}, we assume that the exposure of an item is controlled by the latent exposure strategies, the user, and the item. Based on the above generation process, we first provide the theoretical guarantee of our method via identification analysis. Second, we employ a variational auto-encoder to reconstruct the latent exposure strategies, with the help of the social networks and the items. Third, we devise a counterfactual reasoning based recommendation algorithm by leveraging the recovered exposure strategies. Experiments on four real-world datasets, including three published datasets and one private WeChat Official Account dataset, demonstrate significant improvements over several state-of-the-art methods. 
\end{abstract}
}

\maketitle
\IEEEdisplaynontitleabstractindextext
\IEEEpeerreviewmaketitle

\section{Introduction}
\begin{figure*}[tb]
		\includegraphics[width=\textwidth]{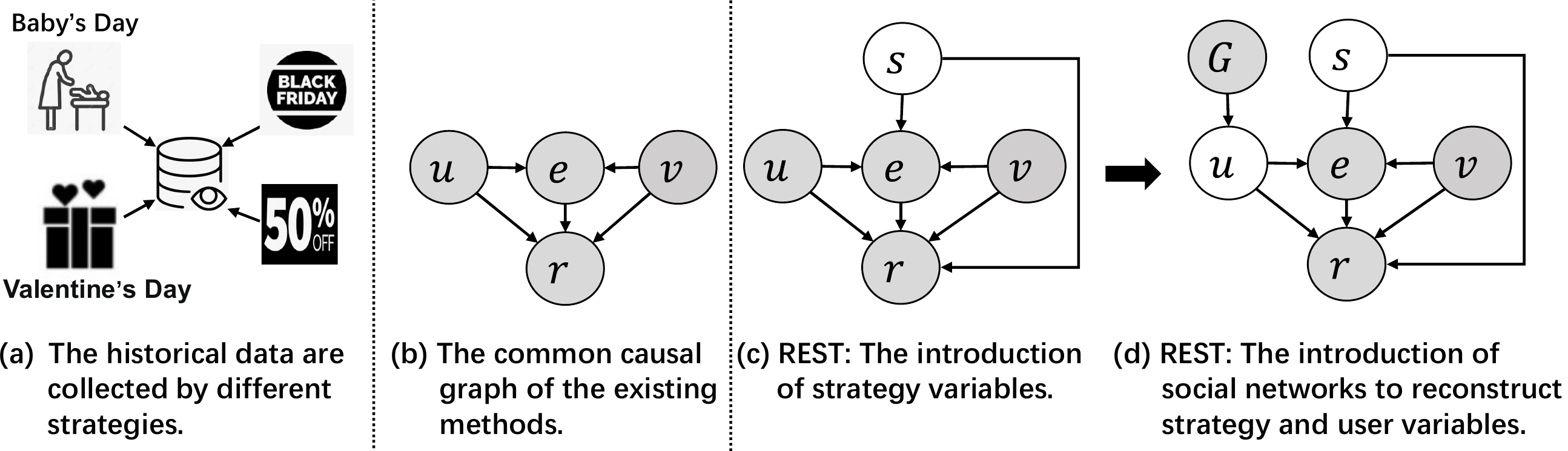}
		\caption{(a) The historical data are collected under different strategies, such as the promotion in Black Friday and the Valentine's Day Gift Shop. (b) The common causal graph of the existing methods, in which implicitly leverage the stationary strategy assumption. \textcolor{black}{(c)(d) are the causal generation process of the proposed method. (c) The causal graph that takes the latent strategy variables into account. (d) In order to address the counterfactual problem in recommendation, we bring the social networks into the causal generation process, where the user variables are latent due to the complexity of the aggregation process of social networks.} }
		\label{fig:motivation}
	\end{figure*}
	Recommendation system \cite{zhang2019deep,zhang2020explainable,1167344,covington2016deep,10.1145/3178876.3186066} is an important techniques in the world. It has been used for a wide range of applications such as e-commerce \cite{ma2019learning,cen2020controllable}, search engines \cite{tennenholtz2019rethinking,baeza2004query} and e-resource services platforms \cite{bouraga2014knowledge}. Recently, the data driven recommendation systems, which use historical data to model the complex relationships among users and items, have achieved a huge success and become mainstream \cite{koren2009matrix,mnih2007probabilistic,he2017neural,jamali2010matrix}.
	
	Selection bias is one of the key issues to the success of the data driven recommendation systems \cite{steck2011item,wang2020click,chen2020bias}. Because the historical data are collected under multiple types of exposure strategies and seriously biased. Give an example in the information flow application, on the one hand, the items recommended by the system will get a high chance to expose and further results in the bias of the collected historical data; on the other hand, users usually prefer to watch and rate the popular videos/news such that the recommended strategies will capture the trending videos/news and further increase the exposures. Such bias of the historical data will lead to overestimate or underestimate the performance of the trained recommendation systems, harm the performance of the deployed recommendation systems and even result in the well-known negative phenomenon named \textit{Information Cocoons} \cite{song2021similar}.
	
	To tackle the aforementioned selection bias problem, many researchers raise several debiased recommendation algorithms \cite{ovaisi2020correcting,schnabel2016recommendations,wang2016learning}. One of the mainstream approaches is the inverse propensity weighting (IPW) based methods \cite{liang2016causal}, e.g., the Empirical Risk Minimization framework \cite{schnabel2016recommendations} and the CausE method \cite{bonner2018causal}. \textcolor{black}{Recently, Feng et.al \cite{wang2021deconfounded,wei2020model,zhang2021causal} employ the concept of causal effect to address the selection bias challenge. }
	By viewing the exposed items ($v$) behind the dataset as the common cause to the exposure ($e$) and the rating level ($r$), we can rephrase the existing methods into the causal graph in Figure \ref{fig:motivation}(b). This figure further explicitly models the user preference, i.e., the user preference will increase the exposure due to the recommended strategies and the rating level based on the user preference.
	
     \textcolor{black}{According to the aforementioned debiased methods}, we can easily find that the core of them is to remove the effect of exposure and to estimate the outcome of \textcolor{black}{rate level if item $v$ is exposed to user $u$} . However, in a large range of historical data, different exposure strategies could change dramatically because the recommendation systems always try to capture the popular trending items or the changing user preference. 
     Take Figure \ref{fig:motivation}(a) as an example, a dataset is a collection of observations over a number of different \textcolor{black}{promotions strategies}, such as Baby's Day, Valentine's Day, Black Friday, and seasonal offer. Such \textcolor{black}{several types of strategies} would render the existing debiased methods fail as they assume a stationary exposure strategy. Furthermore, the strategies are usually independent of the user preference and the item properties, but ignoring the changing strategies will entangle the information of strategies with the users and items, which further leads to bias results.
     Please note that the stationary exposure strategy assumption can be equivalently viewed as the stationary propensity score over time. \textcolor{black}{For example, applying the stationary exposure strategy assumption for reweighting-based methods equals providing the same sample weights for the flowers on Valentine's Day and Baby's Day, which is obviously unreasonable. In order to address this challenge, one straightforward solution is to take the exposure strategies into account. Hence we can obtain the revised causal graph shown in Figure \ref{fig:motivation}(c), in which strategies not only affect the exposure variables but also the rate level.} 
	
	
	\textcolor{black}{Given the causal graph shown in Figure \ref{fig:motivation}(c), the recommendation task can be seen as a counterfactual question that \textit{What the rate level $r$ would be if an item $v$ is exposed to a user $u$ under strategy $s$?} This question is hard to answer since we can only obtain the rate level from the exposed dataset. This difficulty can be solved if we can find a similar user who has given the rate level for the same item. But how to find such a similar user is another nontrivial task with the assumption that each user is independent unless the social networks (or local neighbors) are taken into consideration. So we further propose another revised causal graph as shown in Figure \ref{fig:motivation}(d). It is noted that we let the user variables be latent when taking the social networks into account. This is because the interests of users are influenced by their neighbors, so the user variables become an aggregation of neighbors' information and are too complex to be explicitly described. 
	}
	
	
	Based on the causal graph shown in Figure \ref{fig:motivation}(d), we provide a practical approach for debiased recommendation \textbf{R}econstructing \textbf{E}xposure \textbf{ST}rategies (REST in short.) by modeling different strategies behind observed data. First, we assume that the data generative process of recommendation follows the causal graph shown in Figure \ref{fig:motivation}(d). \textcolor{black}{Second, We summarize the problem of the recommendation systems as the counterfactual reasoning problem and provide the identification analysis for theoretical guarantee.} Third, based on this causal generative process, we devise a variational-based 
	\textcolor{black}{counterfactual reasoning }method to successively reconstruct the user latent variables and the exposure strategy latent variables. Extensive experimental studies demonstrate that the proposed \textbf{REST} method outperforms the state-of-the-art recommendation methods (\textcolor{black}{including the latest methods based on causal effect.}) on three published datasets and one real-world WeChat official accounts dataset.
	
	The rest of the paper is organized as follows. Section \ref{sec:related} reviews existing studies on recommendation systems, including social recommendation systems,  causality-based recommendation systems as well as recommendation systems using a generative model. 
    \textcolor{black}{In section \ref{sec:model}, we expound the causal generation process under latent strategies and social networks. }We also elaborate on the details about how to model the aforementioned causal generation process and how to implement the proposed model in section \ref{sec:imple}. 
    Section \ref{sec:exp} presents the experiment results on four real-world datasets, including ablation analysis and the visualization. Section \ref{sec:conclu} concludes the paper.

\section{Related Works} \label{sec:related}
	Our work is closely related to the recommendation systems in causal view, the social recommendation systems and the recommendation systems that are related to generative models. In this section, we review the works on these three kinds of recommendation systems.
	
	In order to address the problem of selection bias, many researchers borrow the ideas of causal inference \cite{louizos2017causal,liang2016causal}. Sharma et al. \cite{sharma2015estimating} estimate the causal effect of recommendation system from observed data. 
	Schnabel et al. \cite{schnabel2016recommendations} estimate the quality of a recommendation system with the help of the propensity-weighting method which is commonly used in causal inference. 
	Aiming to learn to rank with biased data with click propensities, Ai et al. \cite{ai2018unbiased} propose the Dual Learning Algorithm that combines an unbiased ranker and an unbiased propensity model. 
	Bonner et al. \cite{bonner2018causal} propose the CausE that is optimized with biased logged data and predicts recommendation results under random exposure. 
	Considering that the missing rating in a recommendation system is usually missing not at random, Wang et al. \cite{wang2019doubly} propose a doubly robust estimator for recommendation and further derive the tail bound of the estimator. 
	Recently, Wang et al. \cite{wang2020information} take the unexposed user-item pairs as the counterfactual samples, and propose the counterfactual variational information bottleneck. 
	Motivated from the counterfactual propensity-weighting approach from causal inference, Xu et al. \cite{xu2020adversarial} address the unbiassed recommendation problem by using a minimax empirical risk formulation. However, the aforementioned methods ignore that the historical logged data are collected under different strategies and these strategies are the reasons that lead to selection biases. \textcolor{black}{Moreover, the aforementioned methods implicitly assume that the exposure strategies are stationary and this assumption is usually too strong.} In this paper, we address the selection biases problems in the recommendation by modeling the exposure strategies by combining the social networks with the causal generative process of rate level.
	
	For the social recommendation, one of the goals of recommendation is to learn a better user variables, hence more and more researchers leverage the relationships among users with the consideration of the homophily in the social network. Jamali et al. 
	\cite{jamali2010matrix} combine matrix factorization with the mechanism of trust propagation of social networks in order to address the problems brought by the cold-start users. 
	Following the intuition that personal behaviors are affected by a person's social network, Ma et al. \cite{ma2008sorec} propose SoRec, which learns the user latent feature space and item latent feature space by employing the social networks and the user-item matrix simultaneously. 
	In order to address the data sparsity and cold-start problem, Yang et al. \cite{yang2017social} propose TrustMF, which employs matrix factorization technique to map users into low-dimensional latent feature spaces in terms of their trust relationship.
	With the widespread use of deep learning, many researchers make use of neural networks to improve recommendation algorithm. Considering that the current recommendation largely relies on the initialization of the user and item latent feature vectors, Deng et al. \cite{deng2016deep} use deep learning to determinate the initialization in the matrix factorization for the social recommendation. 
	Considering that the users behave and interact differently in social networks and user-item bipartite graphs, Fan et al. \cite{fan2019deep} raise DASO, which adopts a bidirectional mapping method to transfer users’ information between social domain and item domain. \textcolor{black}{In this paper, since both the user variables and the strategies variables are latent, it is hard to reconstruct them at the same time. Hence we introduce the social networks to reconstruct the user embedding first, then leverage it to reconstruct the strategies variables.}
	
	Other researchers borrow the idea of generative models.
	Zhou et al. \cite{zhou2020recommendation} extend the flow-based generative model \cite{papamakarios2019normalizing} to CF for modeling implicit feedback. And Liang et al. \cite{liang2018variational} combine multinomial likelihoods with collaborative filtering and extend variational auto-encoders \cite{kingma2013auto} to collaborative filtering for implicit feedback. 
	Liu et al. \cite{liu2020deep} consider both local and global structures among users under the Wasserstein auto-encoder frameworks. 
	Recently, graph neural networks attract more and more attention, so some researchers combine the generative models and the graph neural networks. 
	Yu et al. \cite{yu2020enhance} propose a deep adversarial framework based on graph convolutional networks to address the problem of the sparsity of user-item relation and the noisy social relations. 
	In this work, \textcolor{black}{we bring the strategies variables into the structural causal model and tackle the recommendation problem as a counterfactual problem. We follow the paradigm of variational auto-encoders \cite{kingma2013auto} to instantiate the proposed REST method}. 
	
\begin{table}
	\centering
	\caption{Notations and Descriptions.}
	\begin{tabular}{c|c}
	\hline
		\toprule
		\small{Notations}  & Descriptions \\
		\midrule
		$u$ & The user variables as well as the user variables. \\
		\hline
		$v$&  The item variables as well as the item variables. \\
		\hline
		$e,r$&  The exposure variables and the rate level variables. \\
		\hline
		$s$&  The strategy variables. \\
		\hline
		$\mathcal{U},\mathcal{V}$            & The user set and the item set.\\
		\hline
		$\bm{E},\bm{R}$            &  The exposure matrix and rate level matrix\\
		\hline
		$\bm{h}$            &  The features extracted by the models\\
		\hline
		$\mathcal{T}, \mathcal{O}$  & The exposed set and the unexposed set.\\
		\hline
		$G$  & The social network over $U$.\\
		\hline
		$P(\cdot), Q(\cdot)$  & The probability distribution of a random variables.\\
		\hline
		$\mathbf{W}_*, \bm{\Theta}_*$ & The parameters of neural networks. \\
		\hline
		$C(u)$    & The accessed items of user $u$. \\
		\hline
		$N(u)$    & The 1st-order neighbors of $u$. \\
		\hline
		$f(\cdot),g(\cdot), \phi(\cdot)$     & The neural networks based function.  \\
		\hline
		$\mathcal{F}_{u}$               & \makecell[c]{The $\beta$-frequence neighbors item set, containing \\ the items that have been accessed by \\ at least $\beta$ neighbors of $u$.} \\
		\bottomrule
	\end{tabular}
	\label{tab:notation}
	\vspace{2.0em}
\end{table}

\textcolor{black}{\section{Identification of Debiased  recommendation}\label{sec:model}
}
\subsection{Notations}
\textcolor{black}{We first introduce the notations in this paper. Let $\mathcal{U}$ and $\mathcal{V}$ denote the sets of users and items respectively. We further let $\bm{E}$ and $\bm{R}$ denote the exposure matrix and the rate level matrix defined over $\mathcal{U}\cup \mathcal{V}$. $e_{uv}$ is an element of $\bm{E}$, with $e_{uv}=1$ denotes that the item $v$ is exposed to the user $u$ and $e_{uv}=0$ denotes that the item $v$ is not exposed to the $u$.  $r_{uv}$ is an element of $\bm{R}$, which denotes the rate level of $u$ on $v$. Hence we let $\mathcal{T}=\{<u, v, r_{uv}>|e_{uv} = 1\}$ and $\mathcal{O}=\{<u, v, r_{uv}>|e_{uv}=0\}$ be the exposed set and unexposed set respectively. In the social recommendation context, a social networks $G$ is associated with the user set $\mathcal{U}$. 
	With the abuse of notation, we also let $u$, $v$ be the embedding of the corresponding entities and ignore the subscripts of $e_{uv}$ and $r_{uv}$.
	The mathematical notations used in this paper are summarized in Table \ref{tab:notation}. }

\textcolor{black}{\subsection{Causal Generation Process under Exposure Strategies and Social Networks}
Based on the aforementioned notation description, we consider the causal graph to model the recommendation procedure.  As shown in Figure \ref{fig:motivation}(d), the causal graph  contains six variables: $G, u, e, s, v$, and $r$. In particular, we let:
\begin{itemize}
    \item $G\rightarrow u$ denotes how the social networks affect the interests of users.
    \item $u,v,s\rightarrow e$ denotes that whether an item will be recommended depends on $u, v$ and $s$.
    \item $u,e,v,s\rightarrow r$ denotes that the exposure of item $v$ to user $u$ not only depends on $u$ and $v$ but also depends on the exposure strategies $s$.
\end{itemize}}

\textcolor{black}{
Please note that our causal model (Figure \ref{fig:motivation}(d)) is different from the existing debiased method (Figure \ref{fig:motivation}(b)) from the following two aspects: 1) our model further takes the $s$ into consideration. 2) our model takes $u$ as latent variables and employs $G$ as the surrogate of $u$. This causal mechanism provides us a way to infer the latent variables $s$, because $G$ (similarly for $v$) and $s$ are dependent on each other conditioning on $e$. In other words, $G$ and $v$ provide us the clues to infer the latent exposure strategies. 
}

\textcolor{black}{\subsection{Social Recommendation as a Problem of Counterfactual Reasoning}
	Based on the aforementioned descriptions, we provide the definition of the social recommendation.}
	
	\textcolor{black}{We first let $\mathcal{T}'$ be the training set extracted from the exposed set, e.g., $\mathcal{T}'\subseteq \mathcal{T}$. Given the social networks $G$, the training set $\mathcal{T}'$ and the strategy variables $s$, the goal of the social recommendation is to obtain a model that can estimate the following conditional distribution:
	}
	\begin{equation}
	\setlength{\abovedisplayskip}{7pt}
    \setlength{\belowdisplayskip}{7pt}
	\label{defin}
	    P(r|G,u,v,s,do(e=1)),
	\end{equation}
	\textcolor{black}{in which $<u, v, r>$ is extracted from the unexposed set, e.g. $<u, v, r> \in O$. Note that the sample $<u, v, r>$ is extracted from the unexposed set but given $e=1$, meaning that a user $u$ has never been exposed to an item $v$. And estimating the aforementioned conditional distribution equals to answer the following question: \textit{What the rate level $r$ would be if an item $v$ is exposed to a user $u$ given exposure strategy $s$ and social networks $G$?} Therefore, according to the theory of counterfactual inference \cite{pearl2009causality}, we can find that designing a social recommendation system is a counterfactual problem.}

\subsection{Identifying Unbiased Prediction of Social Recommendation System}

	Following the causal view of the recommendation systems, the goal of our social recommendation is to estimate the conditional distribution $P(r|G,u,v,do(e=1))$ according to the \textit{do}-calculus \cite{pearl2009causality}. The identification of such a counterfactual model is an immediate result of Pearl's back-door criteria, as shown in Theorem \ref{the1}.
\begin{theorem}
\label{the1}
\textbf{(Identification of Social Recommendation)} Suppose that the joint distribution $P\left(G, e, r, v, s\right)$ is recovered, the counterfactual prediction is identifiable under the causal model in Figure 1(d).
\begin{proof}
\textcolor{black}{We prove that $ P(r |G ,v ,do(e\!=\!1))$ is identifiable under the premise of the theorem with the help of Equation (\ref{equ1}).}
\begin{equation}
\label{equ1}
\begin{split}
& P(r |G ,v ,do(e=1))\\
= & \!\!\int \!\!\! P(r |G ,u ,v ,s ,do(e \!\!=\!\!1))P(u ,s |G ,do(e \!\!=\!\!1))du ds\\
= & \!\!\int \!\!\!P(r |G ,u ,v ,s ,e \!\!=\!\!1)P(u |G )P(s )du ds,
\end{split}
\end{equation}
where the second equality is based on the rule of do-calculus and conditional independent property under Figure 1(d) \cite{pearl2009causality}. Essentially, we now can predict intervention based on the recovered join distribution $P\left(G, e, r, v,s\right)$, which finishes the proof.
\end{proof}

\end{theorem}

Please note that the aforementioned identification theorem of social recommendation shows that we can estimate the conditional distribution in Equation (\ref{defin}) with the help of social networks and the data extracted from the exposed set $\mathcal{T}$.

	\begin{figure*}[t]
		\includegraphics[width=\textwidth]{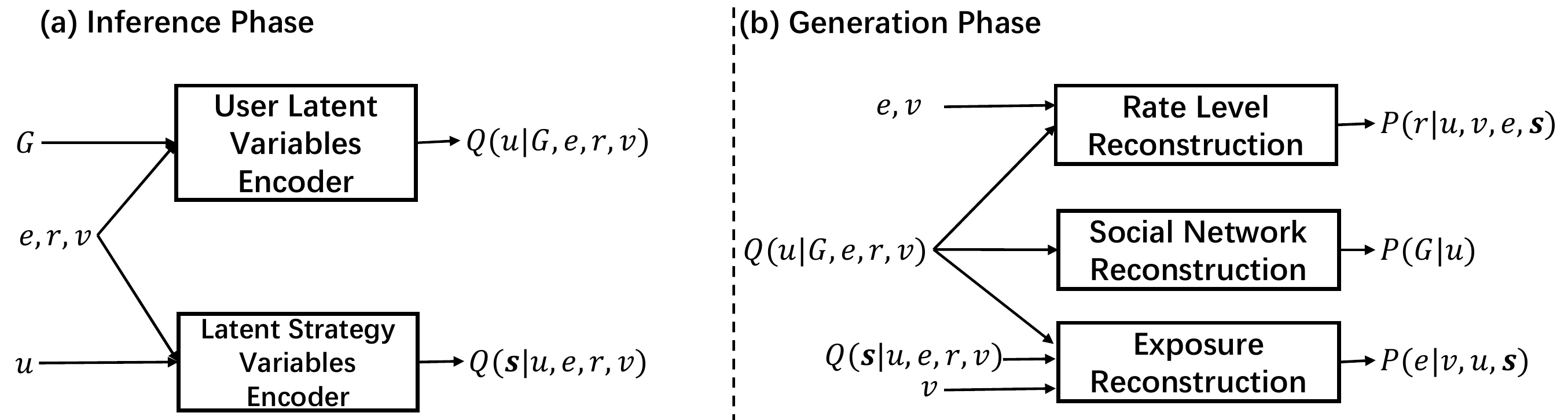}
		\caption{The illustration of the framework of the proposed REST Model. (a) The inference phase contains the Latent User Variable Encoder and the Bias Encoder, which are used to reconstruct the user variables and the strategy variables respectively. (b) The generation phase contains the rate level reconstruction, the social network reconstruction and the exposure reconstruction, which are used to respectively reconstruct the observed rating values, social structures and the exposed variables. In the test step, we only employ the rate level reconstruction for prediction. }
		\label{fig:model}
\end{figure*}
	\section{Algorithm and Implement}\label{sec:imple}
	According to the causal graph shown in Figure \ref{fig:motivation}(d), we devise a variational auto-encoders based framework. We begin with the likelihood of the samples to derive the evidence lower bound (ELBO) of the model. Essentially, the logarithm of joint likelihood $\ln P(G, e, r, v)$ can be written as follows:
	\begin{equation}
	    \begin{split}
	        \ln P(G, e, r, v) &= \mathcal{L}_{ELBO} + \\ D_{KL}&\left(Q(u|G, e, r, v)||P(u|G, e, r, v)\right) +\\ D_{KL}&\left(Q(s|e,r,v, u)||P(s|e,r,v, u)\right),
	    \end{split}
	\end{equation}
	in which the second and third lines are the KL divergence between the approximate distributions and the true posteriors. And $\mathcal{L}_{ELBO}$ is the variational lower bound, which can be derived as follows: (See more details in Supplementary.)
	\begin{equation}
	\label{equ3}
	    \begin{split}
	        \mathcal{L}_{ELBO}  =& P(v)-D_{KL}\left(Q(s|e,r,v, u||P(s)\right) \\
	        &-  D_{KL}\left(Q(u|G, e, r, v)||P(u)\right) \\
	        &+  \mathbb{E}_{Q(u_i|G, e, r, v)}\ln\left(P(G|u)\right)\\
	        &+  \mathbb{E}_{Q(u|G, e, r, v)}\mathbb{E}_{Q(s|e,r,v, u)}\ln P(e|v, u, s)\\
	        &+  \mathbb{E}_{Q(u|G, e, r, v)}\ln P(r|u, v, e),
	    \end{split}
	\end{equation}
	where $Q(u|G, e, r, v)$ and $Q(s|e,r,v,u)$ are the \textcolor{black}{approximate functions that are also respectively named} \textit{user latent variables encoder} and the \textit{latent strategies variables encoder}. These two encoders are used to approximate the two true posteriors: $P(u|G, e, r, v)$ and $P(s|e,r,v,u)$. And we further let $P(G|u)$, $P(e|v, u, s)$ and $P(r|u, v, e)$ denote the \textit{social networks reconstruction}, the \textit{exposure reconstruction} and the \textit{rating level prediction}, respectively. To further facilitate the learning of the model, we assume that the latent user variables follows the delta distribution and the latent strategies variables $s$ follows the categorical distribution. Since these two priors are also consistent with the real-world recommendation systems.
	
	Given the item variables $v$, $P(v)$ is a constant. Moreover, since we assume that $P(u|G, e, r, v)$ is a delta distribution, the value of $D_{KL}\left(Q\left(u|G, e, r, v\right)||P\left(u\right)\right)$ is equal to $0$, the proof is provided in the Proposition \ref{propos1}. 
	
\begin{proposition}
\label{propos1}
\textit{(KL-Divergence under Delta Distribution Assumption)} KL-divergence $D_{KL}(Q^*(u|G, e, r, v)\|P(u))$ is zero if $P(u)$ is a delta distribution with the optimal parameters $Q^*=\arg \max_{Q} \mathcal{L}$.
\end{proposition}
\begin{proof}
We proof by contradiction. First, we suppose that $\\D_{KL}(Q^*(u|G, e, r,v)\|P(u))\neq 0$. Then, given the delta distribution $P(u)$, there must exist an instance $u$ such that $Q^*(u|G, e, r, v)\neq 0,P(u)=0$. It follows that $KL(Q^*(u|G, e, r, v)\|P(u))\to \infty$ leading to an under-optimized score $\mathcal{L}\to -\infty$ which is a contradiction.
\end{proof}

\textcolor{black}{Combing Proposition \ref{propos1} and Equation (\ref{equ3})}, we can reformulate the objective function of the proposed REST model as follows:
	\begin{equation}
	\label{equ:obj}
	\begin{split}
	    \mathcal{L}_{t} &= D_{KL}\left(Q(s|e,r,v, u)||P(s)\right) \\ &
	    - \mathbb{E}_{Q(u|G, e, r, v)}\ln\left(P(G|u)\right) \\ &
	    -\mathbb{E}_{Q(u|G, e, r, v)}\mathbb{E}_{Q(s|e,r,v, u)}\ln P(e|v, u, s) \\ &
	    -\mathbb{E}_{Q(u|G, e, r, v)}\ln P(r|u, v, e).
	\end{split}
	\end{equation}
	
	According to the objective function shown in Equation (\ref{equ:obj}), we can find that the proposed model can be summarized into two phases: the \textit{inference phase} and the \textit{generation phase}, \textcolor{black}{which are} illustrated in Figure \ref{fig:model}. Specifically, the \textit{inference phase}\textcolor{black}{, which is used to infer the latent variables,} is composed of the user latent variable encoder $Q(u|G, e, r, v)$ and the latent strategies variables encoder $Q(s|e,r,v, u)$. The \textit{generation phase}\textcolor{black}{, which is used to infer the observational variables,} is composed of the social network reconstruction $P(G|u)$, the rate level reconstruction $P(r|u, v, e)$ and the exposure reconstruction $P(e|v, u, s)$. We will describe the implementation of the aforementioned components in the following subsections.
%
	

	\subsection{Inference Phase}
	\subsubsection{User Latent Variable Encoder}

 In this part, we first introduce the details of the user latent variable encoder $Q(u|G,e,r,v)$ given the social network $G$, item $v$ as well as the corresponding rate level $r$ and exposure variables $e$. \textcolor{black}{\textcolor{black}{The procedure of inferring the user latent variables} is composed of three steps.} \textcolor{black}{First, we aggregate the information of bipartite graph to obtain the aggregated representation $\bm{h}^b$. Second, we employ a similar way to obtain the aggregated representation $\bm{h}^s$ on social networks. Third, we split the aggregated representation for each type of exposure variables and then process them with different \textcolor{black}{multilayer perceptrons (MLPs)}.}
	
As for the first steps, we need to obtain the aggregated representation of the user-item bipartite graph, we employ the techniques of graph attention networks (GAT) \cite{velivckovic2017graph}. In detail, given the user $u_i$, the interacted item sets $C(u)$ and the corresponding ratings, we obtain the aggregated representation $\bm{h}^b$ \textcolor{black}{with the help of attention mechanism} as follows:
	\begin{equation}
	\centering
	\label{equ:item_aggre}
	    \begin{split}
	    \bm{h}^b=& \sigma\left(\sum_{v \in C(u)}a_{uv}^b \left(v\oplus r\right)\right),\\
	       a_{uv}^b = &\frac{g_b(u,v,r;\bm{W}_b)}{\sum_{v'\in C(u)}g_b(u,v',r;\bm{W}_b)},
	    \end{split}
	\end{equation}
	\textcolor{black}{in which $g_b(\cdot)$ with trainable parameters $\bm{W}_b$ is the score function that is used to calculate the matching score given $u,v,r$; }
	\textcolor{black}{$a_{uv}^b$ denotes the important weights between user $u$ and item $v$.} And $\sigma(\cdot)$ denotes the LeakyReLU, which is the leaky version of a rectified linear unit; $C(u)$ denotes the items list that $u$ has accessed in the bipartite graph.
	
		\begin{figure}[t]
		\includegraphics[width=0.49\textwidth]{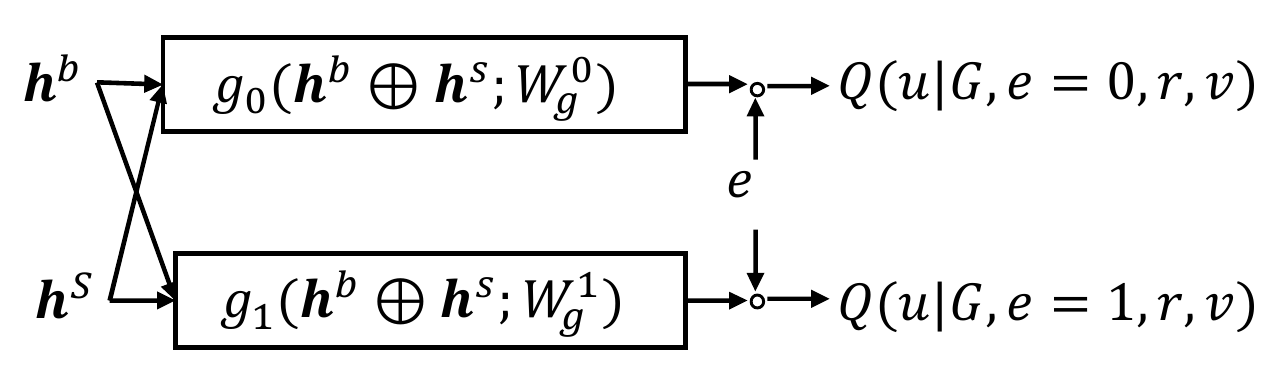}
		\caption{The implementation of the exposure-specific function.}
		\label{fig:model_2}
		\vspace{1.5em}
	\end{figure}
	
	\textcolor{black}{Secondly, we use another GAT to obtain the aggregated representation of social networks. Practically, we let $G$ in $Q(u|G,e,r,v)$ be the substructure of the social networks, for example, the 1st-order neighbors of $u$. The calculation procedure is shown as follows:}
	\begin{equation}
	    \begin{split}
	    \bm{h}^s =& \sigma\left(\sum_{k \in N(u)}a_{uk}^s\cdot u\right),\\
	       a_{uk}^s=&\frac{g_s(u,k;\bm{W}_s)}{\sum_{u'\in N(u)}g_s(u,u';\bm{W}_s)}
	        ,
	    \end{split}
	\end{equation}
    \textcolor{black}{in which $g_s(\cdot)$ with trainable parameters $\bm{W}_s$ is the score function that is used to calculated the matching score given any two user embedding;}
	\textcolor{black}{$a_{uk}^s$ denotes the important weight between $u$ and $k$; And $N(u)$ denotes the 1st-order neighbors of $u$.}

	\textcolor{black}{Finally, in order to obtain the user latent variables, we devise the exposure-specific architecture inspired by CEVAE \cite{louizos2017causal} and TARnet \cite{shalit2017estimating}, which is shown in Figure \ref{fig:model_2}. we use $g_{0}(\cdot),g_{1}(\cdot)$ to generate the user latent variables $\bm{h}_{u}$. Specifically, we can obtain the latent user variables via the exposure-specific functions as follows:}
	\begin{equation}
	\label{equ:exposure}
	        \bm{h}_{u} = (1-e) * g_{0}(\bm{h}^b, \bm{h}^s;\bm{W}_{g}^{0}) + e * g_{1}(\bm{h}^b, \bm{h}^s;\bm{W}_{g}^{1}),
	\end{equation}
	where $g_{0}(\cdot)$ and $g_{1}(\cdot)$ are composed of MLPs, $\bm{W}_{g}^{0}$ and $\bm{W}_{g}^{1}$ are the trainable parameters. For convenience, we let $\bm{W}_g=\{\bm{W}_g^{0}, \bm{W}_g^{1},\bm{W}_b,\bm{W}_s\}$. For a tetrad $(u, v, r,e=1)$ in the exposed set, we use $g_{1}(\bm{h}^b \oplus \bm{h}^s;\bm{W}_g^1)$. For those from the unexposed set, we use $g_{0}(\bm{h}^b \oplus \bm{h}^s;\bm{W}_g^0)$.
	
	\label{unexposed}
	The training of the unexposed $g_{0}(\cdot)$ is crucial to the success of our counterfactual learning problem \cite{louizos2017causal}. The main challenge is that we can only obtain the rate levels on the exposed set from the bipartite graph and the rate levels on the unexposed set are unavailable. To address this challenge, we approximate the unexposed set in the following three steps:
	\begin{itemize}
	    \item First, we extract $C(u)$ and $G_u^1$ for the user $u$, where $G_u^1$ is the set of 1-order neighbors of $u$.
	    \item Second, we extract the $\beta$-frequency neighbors item set $\mathcal{F}_{u}$ by $\mathcal{F}_{u}=\{v|v \notin C(u), \sum_{u' \in N(u)}e_{u'v}\geq \beta\}$.
	    \item Third, we obtain the unexposed sample $(u, v, r, e=0)$, in which $v \in \mathcal{F}_{u}$ and $r$ is the most frequent rate level of $u$'s neighbors, i.e, using the voting method to get the value $r$ for the unexposed samples.
	\end{itemize}
	\textcolor{black}{Please note that the aforementioned procedure to generate the counterfactual samples implicitly leverage the assumption that \textit{both the users and their friends share similar interests and behaviors.}} 
	
	\subsubsection{Latent Strategies Variables Encoder}
	In this subsection, we aim to model the latent strategies variables $s$ by using the exposure variables $e$, item variables $v$, rate level $r$ and latent user variables $u$. First, we follow the same aggregation method in Equation (\ref{equ:item_aggre}) to calculate the item aggregated representation $\bm{h}^d$ for discrete strategies variables encoder, which is shown as follows:
	\begin{equation}
	    \begin{split}
	        \bm{h}^d =& \sigma\left(\sum_{v \in C(u)}a_{uv}^d \left(v\oplus r\right)\right),\\
	       a_{uv} =& \frac{g_d(u,v,r;\bm{W}_d)}{\sum_{v\in C(u)}g_d(u,v',r;\bm{W}_d)},
	    \end{split}
	\end{equation}
	in which $\bm{W}_d$ are the trainable parameters. Similar to Equation (\ref{equ:exposure}), we model the latent strategies variable with the help of another exposure-specific function as show in Equation(\ref{equ:gumbel}). Note that we use the Gumbel-Softmax trick \cite{jang2016categorical} to estimate latent strategies variable since we assume they follow the categorical distribution.
	\begin{equation}
	\label{equ:gumbel}
	\begin{split}
	    \hat{s} = (1-e) *& \phi_{0}(\bm{h}^d , v;\bm{W}_{\phi}^{0}) + e * \phi_1(\bm{h}^d , v;\bm{W}_{\phi}^{1}),\\
	    s \sim &\text{GUMBLE-SOFTMAX}(\hat{s}),
	\end{split}
	\end{equation}
	in which $\phi_{0}$ and $\phi_{1}$ are the exposure-specific function in latent strategies variables encoder. $\bm{W}_{\phi}^{(0)}$ and $\bm{W}_{\phi}^{(1)}$ are the trainable variables. For convenience, we let $\bm{W}_s=\{\bm{W}_\phi^{0}, \bm{W}_\phi^{1}, \bm{W}_d\}$. 
	\subsection{Generation Phase}
	
	\subsubsection{Social Networks Reconstruction}
	After obtaining the aforementioned two kinds of latent variables, we aim to reconstruct the social networks. In this paper, we follow the configuration of variational graph auto-encoders \cite{kipf2016variational} and reconstruct each edge of social network structures of $u$ as follows:
	\begin{equation}
	    \hat{G}_{u,u'} = \sigma\left(\bm{h}_u, {\bm{h}_{u'}}\right),
	\end{equation}
	where $\hat{G}_{u,u'}$ is the predicted edge between $u$ and $u'$. \textcolor{black}{In order to train the model with the mini-batch, we only reconstruct the 1-st order neighbors of $u$ instead of the whole social networks.}
	
	\subsubsection{Exposure Reconstruction}
	Given the latent user variables $u$, latent strategies variables $s$ and item variables $v$, we aims to model $P(e|v, u, s)$. We employ the following function to reconstruct the exposure variables:
	\begin{equation}
	   \hat{e} = f_e(\bm{h}_u, v, s;\bm{\theta}_e),
	\end{equation}
	in which $\bm{\theta}_e$ are the trainable parameters and $f_e(\cdot)$ is a neural architecture that is composed of MLPs.
	
	\subsubsection{Rate Level Reconstruction}
	Finally, we aims to predict the rate level, given the user latent variables $u$, item variables $v$ and exposure variables $e$.
	Similar to Equation (\ref{equ:exposure}), we employ the exposure-specific rate level predictor, which is shown as follows:
	\begin{equation}
	    \begin{split}
	        \hat{r}_{uv} &=e * f_{1}(\bm{h}_u,v;\bm{\theta}_{1}) + (1-e) * f_0(\bm{h}_u,v;\bm{\theta_0}),
	    \end{split}
	\end{equation}
	in which $f_{0}$ and $f_{1}$ are also composed of MLPs and $\bm{\theta}_{0}$ and $\bm{\theta}_{1}$ are the trainable parameters. For convenience, we let $\bm{\Theta}_r = \{ \bm{\theta}_{0}, \bm{\theta}_{1}, \bm{\theta}_e\}$.

	\subsection{Model Summarization}
    After combining the inference phase and the generation phase, we summarize the total loss of the proposed method as follows:
	\begin{equation}
	    \mathcal{L}(\bm{W}_g, \bm{W}_s, \bm{\Theta}_e) = \mathcal{L}_t + \gamma \mathcal{L}_{reg},
	\end{equation}
	where $\mathcal{L}_{reg}$ is the L2 regularization of the parameters; $\gamma$ is the hyper-parameter.
	
	During the training step, we optimize the model by using the following procedure:
	\begin{equation}
	   \begin{split}
	        (\hat{\bm{W}}_g, \hat{\bm{W}}_s, \hat{\bm{\Theta}}_r) = \mathop{\arg\min}_{\bm{W}_g, \bm{W}_s, \bm{\Theta}_r} \mathcal{L}(\bm{W}_g, \bm{W}_s, \bm{\Theta}_e).
	   \end{split}
	\end{equation}
	
	During the evaluation step, given the $u$ and unexposed item $v$, we let $e=1$. The following procedure with the trained optimal parameters is adapted to the test dataset.
	\begin{equation}
    \begin{split}
        \bm{h}_{u} =& g_{1}(\bm{h}^b , \bm{h}^s;\bm{W}_{g}^{1}),\\
        \hat{r} =& f_{1}(\bm{h}_u,v;\bm{\theta_{1}}).
    \end{split}
	\end{equation}

\section{Experiment}\label{sec:exp}
	In this section, we report experimental results on four datasets to evaluate our method against the state-of-the-art baselines, \textcolor{black}{including the latest methods that use the idea of causal effect}. \textcolor{black}{With the help of the experiment results, we want to explore the following challenges: (1) Can the proposed REST method remove the disadvantage of selection biases? How is the performance compared with the existing methods? Especially the causality-based methods. (2) Can the strategies variables in the proposed REST method model effectively mitigate the non-stationary strategies challenges?}

	\subsection{Datasets}
	In order to evaluate the performance of our method, we conduct experiments on three published datasets (including Ciao, Epinions, and Yelp) with explicit feedback and a private dataset collected from WeChat official accounts with implicit feedback. The details of the aforementioned dataset are shown in Table \ref{tab:dataset}.
	
\begin{table}[tb]
\caption{Statistics of the datasets}
\centering
\scalebox{1.1}{\begin{tabular}{|l|l|l|l|l|}
\hline
\makecell[c]{Dataset}                    & Epinions & Yelp      & Ciao & WeChat    \\\hline
\makecell[c]{\# of users }               & 22K   & 332K   & 7K  & 568K   \\\hline
\makecell[c]{\# of items}                & 296K  & 197K   & 106K     & 242K   \\ \hline
\makecell[c]{\# of user-item \\ relationship} & 798K  & 4,567K & 282K     & 9,422K \\\hline
\makecell[c]{\# of user-user \\ relationship} & 355K  & 7,043K & 57K     & 5,667K \\\hline
\makecell[c]{Social networks \\ density}    & 0.072\%  & 0.0064\%  & 0.11\%     & 0.0018\%  \\\hline
\makecell[c]{Bipartite graph \\ density}    & 0.012\%  & 0.0070\%  & 0.036\%     & 0.0068\% \\\hline
\end{tabular}}
\label{tab:dataset}
\vspace{1.5em}
\end{table}
	
	\begin{itemize}
\item {Ciao}$\footnote{www.ciao.co.uk}$ is a published dataset for the social recommendation. The source cite of Ciao allows users to rate items, and add friends to their ‘Circle of Trust’.
\item {Epinions}$\footnote{http://www.trustlet.org/extended\_epinions.html}$: A benchmark dataset for the social recommendation. In Epinions, a user can rate and give comments on items. Besides, a user can also select other users as their trusters. Note that we treat the trust graphs as social networks. 
\item {Yelp}$\footnote{https://www.kaggle.com/yelp-dataset/yelp-dataset}$: An online review platform where users review local businesses (e.g., restaurants and shops). The user-item interactions and the social networks are extracted in the same way as Epinions.  
\item {WeChat Official Accounts Dataset}: WeChat is a Chinese multi-purpose messaging, social media, and mobile payment application developed by Tencent. And WeChat official accounts dataset is one of the functions. On the WeChat Official Account platform, users can read and share articles. This dataset is constructed by user-article clicking records and user-user social networks on this platform. 
\end{itemize}
For each dataset, we split it into the training set, validation set, and test set. We choose the model with the best validation and evaluate the chosen model on the test set. Note that we do not consider new users and new items in validation and testing. All the methods run with five different random seeds, and we report both the mean and variance. The source code and the prepossessing scripts of the proposed methods are available at the following link$\footnote{ https://github.com/DMIRLAB-Group/REST}$.

\subsection{Hyper-parameters}
\textcolor{black}{We optimize all models with the Adam optimizer with the batch size of 1024. For a fair comparison, all the methods are fine-tuning by searching the learning rate in the range of $\{0.001, 0.0009, \cdots, 0.0001\}$. We also adopt the early stopping strategy that stops training if RMSE/HR@20 on the validation dataset does not decrease/increase for 1500 training steps.}

\subsection{Evaluation Metrics}
	We use different evaluation metrics for datasets with explicit feedback and implicit feedback respectively. 
	
	For the dataset with explicit feedback, we use MSE and RMSE. 
	The smaller values of MAE and RMSE, the better the predictive accuracy is. Note that even a small improvement in RMSE or MAE terms can have a significant impact on the quality of the top-few recommendations.
	
	For the dataset with implicit feedback, we use Hit Rate@$K$ (HR@$K$) and Normalized Discounted Cumulative Gain@$K$ (NDCG@$K$).
HR measures the percentage that recommended items contain at least one correct item interacted by the user, while NDCG takes the positions of correct recommended items into consideration. In this paper, we choose $K$ in $\{5, 10, 20\}$. Note that higher scores of HR@$K$ and NDCG@$K$ indicate better performance.


	\subsection{Baselines}
	We compare the proposed \textbf{REST} method with four kinds of baselines. Besides the classical matrix factorization based Methods, we also take some graph neural networks based methods into account. Furthermore, we also compare our method with the baselines based on causal inference. Since our method uses the technique of variational influence, we also consider some VAE based methods.
	
	\begin{table*}
\caption{The performance evaluation of the compared methods on Ciao, Epinion, Yelp dataset. The value presented are averaged over 5 replicated with different random seeds. Standard deviation is in the subscript.}
\label{tab:explicit_fb}
\centering
\begin{tabular}{|c|c|cc|cc|cc|} 
\hline
                                &            & \multicolumn{2}{c|}{CIAO} & \multicolumn{2}{c|}{EPINIONS} & \multicolumn{2}{c|}{YELP}  \\ 
\hline
                                & Algorithms & MAE    & RMSE             & MAE & RMSE                    & MAE & RMSE                 \\ 
\hline
\multirow{2}*{MF
  based}     & PMF        & 0.9539 $_{\pm0.0040}$ & 1.1936$_{\pm0.0019}$& 1.0767$_{\pm0.0035}$ & 1.2755$_{\pm0.0022}$& 0.9896$_{\pm0.0023}$ &  1.2454$_{\pm0.0011}$                    \\ 
                                & NeuMF      &0.7770 $_{\pm0.0077}$        & 0.9828$_{\pm0.0022}$                  & 0.8457$_{\pm0.0053}$    &  1.0838$_{\pm0.0015}$                       & 0.9575$_{\pm0.0081}$     & 1.1958$_{\pm0.0005}$                     \\ 
\hline
\multirow{2}{*}{VAE
  based}    & MultiVAE   &  0.9254$_{\pm0.0025}$      &  1.1908$_{\pm0.0014}$                &  0.9707$_{\pm0.0104}$   &       1.2104$_{\pm0.0039}$                   & 0.9957$_{\pm0.0031}$    & 1.2944$_{\pm0.0020}$                    \\ 
                                & RecVAE     &   0.9449$_{\pm0.0014}$     &  1.1787$_{\pm0.0022}$                & 0.9614$_{\pm0.0087}$    &    1.1946$_{\pm0.0038}$                     & 0.9944$_{\pm0.0020}$    &  1.2385$_{\pm0.0014}$                    \\ 
\hline
\multirow{6}{*}{Causality-based} 
& CausE      &   0.7943$_{\pm0.0014}$     &  1.0003$_{\pm0.0013}$                &  0.8553$_{\pm0.0019}$   &   1.0705$_{\pm0.0013}$                      &  0.9400$_{\pm0.0031}$   &  1.2039$_{\pm0.0015}$                    \\ 
                                & CVIB-MF    &  0.9091$_{\pm0.0016}$      & 1.2001$_{\pm0.0011}$                 &  0.9499$_{\pm0.0031}$   &   1.2477$_{\pm0.0003}$                      & 0.9919$_{\pm0.0122}$    &  1.3189$_{\pm0.0024}$                    \\ 
                                & CVIB-NCF   &  0.7394$_{\pm0.0027}$      & 1.0462$_{\pm0.0013}$                  &  0.8311$_{\pm0.0128}$   &  1.2477$_{\pm0.0003}$                       & 0.9801$_{\pm0.0011}$    & 1.3613$_{\pm0.0043}$                     \\ 
                                & MACR-MF    &   0.9446$_{\pm0.0051}$     & 1.1859$_{\pm0.0030}$                 &  0.9784$_{\pm0.0092}$   &   1.2364$_{\pm0.0031}$                      &  0.9923$_{\pm0.0004}$    &  1.2344$_{\pm0.0004}$                    \\ 
                                & DecRS      &  0.7576$_{\pm0.0038}$      &  0.9875$_{\pm0.0033}$                &  0.8242$_{\pm0.0043}$   &   1.0617$_{\pm0.0033}$                      &  -    & -                     \\ 
\hline
\multirow{3}{*}{GNN
  based}    & GraphRec   &  0.7585$_{\pm0.0051}$      &  0.9743$_{\pm0.0021}$                &  0.8283$_{\pm0.0019}$   &       1.0567$_{\pm0.0019}$                  & 0.9525$_{\pm0.0035}$    &  1.1968$_{\pm0.0017}$                    \\ 
                                & NGCF       &    0.8061$_{\pm0.0023}$    & 1.0135$_{\pm0.0010}$                  &  0.9348$_{\pm0.0023}$    &  1.1286$_{\pm0.0017}$                       &  0.9396$_{\pm0.0023}$   & 1.2231$_{\pm0.0017}$                     \\ 
                                & LightGCN   &  0.9373$_{\pm0.0051}$      &  1.1919$_{\pm0.0014}$                &  0.9584$_{\pm0.0011}$   &     1.2025$_{\pm0.0005}$                       &  1.0015$_{\pm0.0024}$   &1.2444$_{\pm0.0019}$                      \\ 
\hline
Ours                            & REST       & \textbf{0.7320}$_{\pm0.0117}$       &      \textbf{0.9635}$_{\pm0.0009}$        &  \textbf{0.8013}$_{\pm0.0045}$   & \textbf{1.0413}$_{\pm0.0007}$                        & \textbf{0.9158}$_{\pm0.0054}$    &    \textbf{1.1733}$_{\pm0.0006}$                  \\
\hline
\end{tabular}
\end{table*}
	
	\textbf{Matrix Factorization based Methods}:
	\begin{itemize}
	\item PMF \cite{mnih2007probabilistic}: Probabilistic Matrix Factorization is one of the most traditional methods for the recommendation that models latent factors of users and items by Gaussian distributions.
	\item NeuMF \cite{he2017neural}: Neural network based Collaborative Filtering replaces the inner product with a neural architecture that can learn an arbitrary function from data. 
	\item BPRMF \cite{rendle2009bpr}: BPRMF which is optimized by stochastic gradient descent with bootstrap sampling, is the maximum posterior estimator that derived from the Bayesian theorem.
	\end{itemize}
	
	\textbf{Graph Neural Networks based Methods}:
	\begin{itemize}
	    \item GraphRec \cite{fan2019graph}: A graph neural networks based method that leverages graph attention mechanism to aggregate the information of the social networks and user-item relations.
	    \item LightGCN \cite{he2020lightgcn}: LightGCN optimizes the user and item representation by linearly propagating them on the bipartite graph, and uses the weighted sum of the representation.
	    \item NGCF \cite{wang2019neural}: NGCF integrates the user-item interactions by modeling the high-order connectivity and injecting the collaborative signal into the embedding process.
	\end{itemize}
	\textbf{Variational Auto-Encoder based Methods}:
	\begin{itemize}
	    \item MultVAE \cite{liang2018variational}: MultVAE extends VAE to collaborative filtering for implicit feedback, so it performs worse on the dataset with explicit feedback.
	    \item RecVAE \cite{shenbin2020recvae} uses the multinomial likelihood variational auto-encoders to map user feedbacks to user embeddings.
	\end{itemize}
	\textbf{Causality-based Methods}:
	\begin{itemize}
	    \item CausE \cite{bonner2018causal}: CausE jointly learns two CTR models and uses a multi-task objective that factorizes the matrix of observations.
	    \item CVIB \cite{wang2020information}: CVIB learns a balanced model based on Information Bottleneck, which simultaneously optimizes the factual and counterfactual embeddings. In this paper, we compare our method with two variants of CVIB: MF-CVIB and NCF-CVIB.
	    \item \textcolor{black}{MACR-MF \cite{wei2020model}: MACR-MF leverages the idea of causal effect and builds a multi-task learning schema over MF. We compare MACR-MF with our method in the implicit feedback dataset. }
	    \item \textcolor{black}{DecRS \cite{wang2021deconfounded}: Deconfounded Recommender System (DecRS) models the causal effect of user representation on the prediction score, which eliminates the impact of the confounder with the help of backdoor adjustment. Note that we only compare DecRs in the Ciao and Epinions datasets, since this method needs the categories of items and the Yelp dataset does not contain the item categories.} 
	\end{itemize}

\begin{table*}[hbpt]
\caption{The performance evaluation of the compared methods on WeChat dataset. The value presented are averaged over 5 replicated with different random seeds. Standard deviation is in the subscript.}
\centering
\scalebox{1.1}{\begin{tabular}{|l|lcccccc|}
\hline
Model Class & \multicolumn{1}{l|}{Models}            & HR@5           & NDCG@5         & HR@10          & NDCG@10        & HR@20          & {NDCG@20}       \\ \hline
& \multicolumn{1}{l|}{BPRMF}                        & 56.16$_{\pm0.16}$ & 44.43$_{\pm0.14}$ & 67.11$_{\pm0.13}$ & 47.98$_{\pm0.14}$ & 77.47$_{\pm0.16}$ & 50.60$_{\pm0.13}$  \\
& \multicolumn{1}{l|}{NeuMF}                        & 60.76$_{\pm0.33}$ & 48.20$_{\pm0.45}$ & 71.18$_{\pm0.32}$ & 51.57$_{\pm0.42}$ & 81.27$_{\pm0.59}$ & 54.13$_{\pm0.49}$  \\\hline
 & \multicolumn{1}{l|}{SocialMF}                     & 41.78$_{\pm0.16}$ & 32.57$_{\pm0.29}$ & 50.01$_{\pm0.18}$ & 35.23$_{\pm0.29}$ & 58.23$_{\pm0.14}$ & 37.31$_{\pm0.25}$  \\
\multirow{-6}{*}{\begin{tabular}[c]{@{}l@{}}MF based\end{tabular}}
& \multicolumn{1}{l|}{MultVAE}                     & 53.15$_{\pm0.07}$ & 41.54$_{\pm0.06}$ & 64.97$_{\pm0.08}$ & 45.36$_{\pm0.07}$ & 76.59$_{\pm0.08}$ & 48.30$_{\pm0.07}$ \\
\multirow{-3}{*}{\begin{tabular}[c]{@{}l@{}}VAE based\end{tabular}}
& \multicolumn{1}{l|}{RecVAE}                         & 55.54$_{\pm0.08}$ & 43.55$_{\pm0.06}$ & 67.34$_{\pm0.09}$ & 47.37$_{\pm0.06}$ & 78.64$_{\pm0.06}$ & 50.23$_{\pm0.05}$  \\ \hline
& \multicolumn{1}{l|}{GraphRec}                     & 54.61$_{\pm1.63}$ & 42.47$_{\pm1.29}$ & 66.57$_{\pm1.74}$ & 46.35$_{\pm1.32}$ & 77.77$_{\pm1.58}$ & 49.18$_{\pm1.28}$   \\
& \multicolumn{1}{l|}{LightGCN}                     & - & - & - & - & - & -  \\
& \multicolumn{1}{l|}{NGCF}                     & - & - & - & - & - & -  \\
\hline
\multirow{-5}{*}{\begin{tabular}[c]{@{}l@{}}GNN based\end{tabular}}
& \multicolumn{1}{l|}{MF-CVIB}                     & 62.28$_{\pm0.26}$ & 50.14$_{\pm0.27}$ & 72.40$_{\pm0.21}$ & 53.43$_{\pm0.25}$ & 81.49$_{\pm0.18}$ & 55.73$_{\pm0.24}$ \\
\multirow{-0.5}{*}{\begin{tabular}[c]{@{}l@{}}Causality-based\end{tabular}}
& \multicolumn{1}{l|}{NCF-CVIB}                         & 63.89$_{\pm0.70}$ & 52.27$_{\pm0.98}$ & 72.84$_{\pm0.49}$ & 55.17$_{\pm0.91}$ & 81.18$_{\pm0.46}$ & 57.28$_{\pm0.87}$ \\
& \multicolumn{1}{l|}{MACR-MF}                         & 61.81$_{\pm0.07}$ & 48.64$_{\pm0.23}$ & 72.34$_{\pm0.30}$ & 52.06$_{\pm0.11}$ & 81.23$_{\pm0.46}$ & 54.32$_{\pm0.07}$ \\
& \multicolumn{1}{l|}{CausE}                         & 59.54$_{\pm0.31}$ & 41.54$_{\pm0.06}$ & 64.97$_{\pm0.08}$ & 45.36$_{\pm0.07}$ & 76.59$_{\pm0.08}$ & 48.30$_{\pm0.07}$ \\ \hline
Ours 
& \multicolumn{1}{l|}{REST}                      & \textbf{65.31}$_{\pm0.27}$ & \textbf{52.07}$_{\pm0.19}$ & \textbf{76.05}$_{\pm0.20}$ & \textbf{55.56}$_{\pm0.19}$ & \textbf{85.37}$_{\pm0.13}$ & \textbf{57.92}$_{\pm0.14}$  \\ \hline 
\end{tabular}}
\label{tab:wechat}
\end{table*}

\subsection{Deconfounding Performance}
\textcolor{black}{In this subsection, we aim to answer (1) Can the proposed REST method remove the disadvantage of selection biases? And how is the performance compared with the existing methods, including the latest causality-based method?}

\subsubsection{Experiment results on datasets with explicit feedback} \textcolor{black}{We first illustrate the experiment results on the explicit feedback dataset, in which the users provide the rating for items. Hence we follow \cite{fan2019graph} and employ the MAE and RMSE as the evaluation metric.} The experiment results on Ciao, Epinions, and Yelp dataset are shown in Table \ref{tab:explicit_fb}. 
From experiment results, we can obtain the following observations: 
\textcolor{black}{\begin{itemize}
    \item The proposed REST method outperforms the other methods with a large margin, which proves that our method can effectively remove the disadvantages of the selection biases. Furthermore, the superior performance of the proposed REST reflects the advantages of the identification theorem.
    \item According to the Table \ref{tab:explicit_fb}, the proposed REST achieves different degrees of improvement on the three explicit datasets. In detail, the REST respectively obtains $26.4\%$, $18.1\%$ and $7.7\%$ improvements on the Ciao, Epinion and Yelp datasets. This is because the social networks densities of these datasets are different. According to Table \ref{tab:dataset}, we can find that the Ciao dataset contains the densest social networks while the Yelp contains the sparsest one. This is because the denser social networks can provide more counterfactual samples, which further benefit the model performance.
    \item Our method not only outperforms the conventional recommendation algorithms like PMF and NeuralMF but also outperforms the VAE-based methods like MultiVAE and RecVAE. This is because the VAE-based methods assume that the distributions of latent variables follow the Gaussian distribution but the assumption is too strong and does not work in practice. In the meanwhile, assuming that the latent variables follow the delta distribution, the proposed REST method can avoid the aforementioned drawback.
    \item The graph neural networks based methods like the LightGCN and the GraphRec, which are designed for the implicit feedback datasets, perform poorly in the explicit feedback dataset. For one thing, this verifies that the graph neural networks are still poisoned by the selection biases even though they leverage the social networks. For another thing, the proposed method leverage the social to generate the counterfactual samples can mitigate the selection biases to some extent.
    \item As for the causal inference based method, our method outperforms the causality based method like CausE, MACR-MF and DecFM. This is because the proposed REST method models latent strategy variables that break the stationary strategy assumption. We will further explore the effectiveness of the latent strategy variables in the following subsections. 
\end{itemize}}


\subsubsection{Experiment results on the dataset with implicit feedback} \textcolor{black}{Then we further illustrate the experimental results on the implicit feedback dataset, in which only the actions of users like clicking or purchasing, are collected. Hence we follow \cite{wei2020model} and employ HR@K and NDCG@K as the evaluation metrics. The implicit feedback scenario is more challenging, because it is hard to distinguish if the unseen samples are disliked or not.} 
The experiment results on the WeChat Offical Account dataset are shown in Table \ref{tab:wechat}. We do not report the experiment result of LightGCN and NGCF because of the limited CPU memory. According to the experiment results, we can get the following conclusions:  
\textcolor{black}{\begin{itemize}
    \item As similar to the experiment results on the explicit feedback dataset, we can find that the proposed REST method still achieves the best performance, which reflects that our method can work on both the explicit and the implicit scenarios.
    \item Compared with the causal based methods like MACR-MF \cite{wei2020model} and the other types of methods, the causality-based methods achieve a better result, which reflects that the selection biases really harm the performance and taking causality into consideration will ease the disadvantage to some extent. 
    \item In the meanwhile, our method also achieves good results. This is because the Wechat dataset is more likely controlled by different types of strategies like different fast-breaking news. Therefore, taking the stationary assumption and ignoring the strategies will degenerate the performance of the recommendation systems even the selection biases have been taken into account.
\end{itemize}}

\subsection{Ablation Analysis}

In order to evaluate the effectiveness of the latent unobserved strategy variables, we raise a model invariant named \textbf{REST-S}, which removes the latent unobserved strategy variables in the data generation process. In this case, we follow the stationary strategy assumption and do not model the strategies behind the data.

\subsubsection{The effectiveness of the discrete strategy variables}
In order to verify the effectiveness of discrete exposure component of our model, we devise \textbf{REST-S}. The experiment results are shown in Figure \ref{fig:mes_ablation} and \ref{fig:ndcg_ablation}. Based on the experiment results, we can observe that:
\textcolor{black}{\begin{itemize}
    \item Compared DUSE-S with the standard REST, we can find that the performance of REST-S is lower that of REST, which reflects the advantages of modeling the discrete strategies.
    \item Since we do not model the discrete strategies in REST-S, both the REST-S and other causality based methods like MACR-MF are the same from the view of principle, so it is reasonable to guess that the performance of both the REST-S and other causality-based methods are similar. Compared REST-S with the other baselines, like CVIB-NCF and MACR-MF, we can find that we still obtain a comparable performance, which not only validates the aforementioned guess but also the effectiveness of modeling strategies.
\end{itemize}}


\begin{figure}[t]
\scalebox{0.49}
{\includegraphics[width=\textwidth]{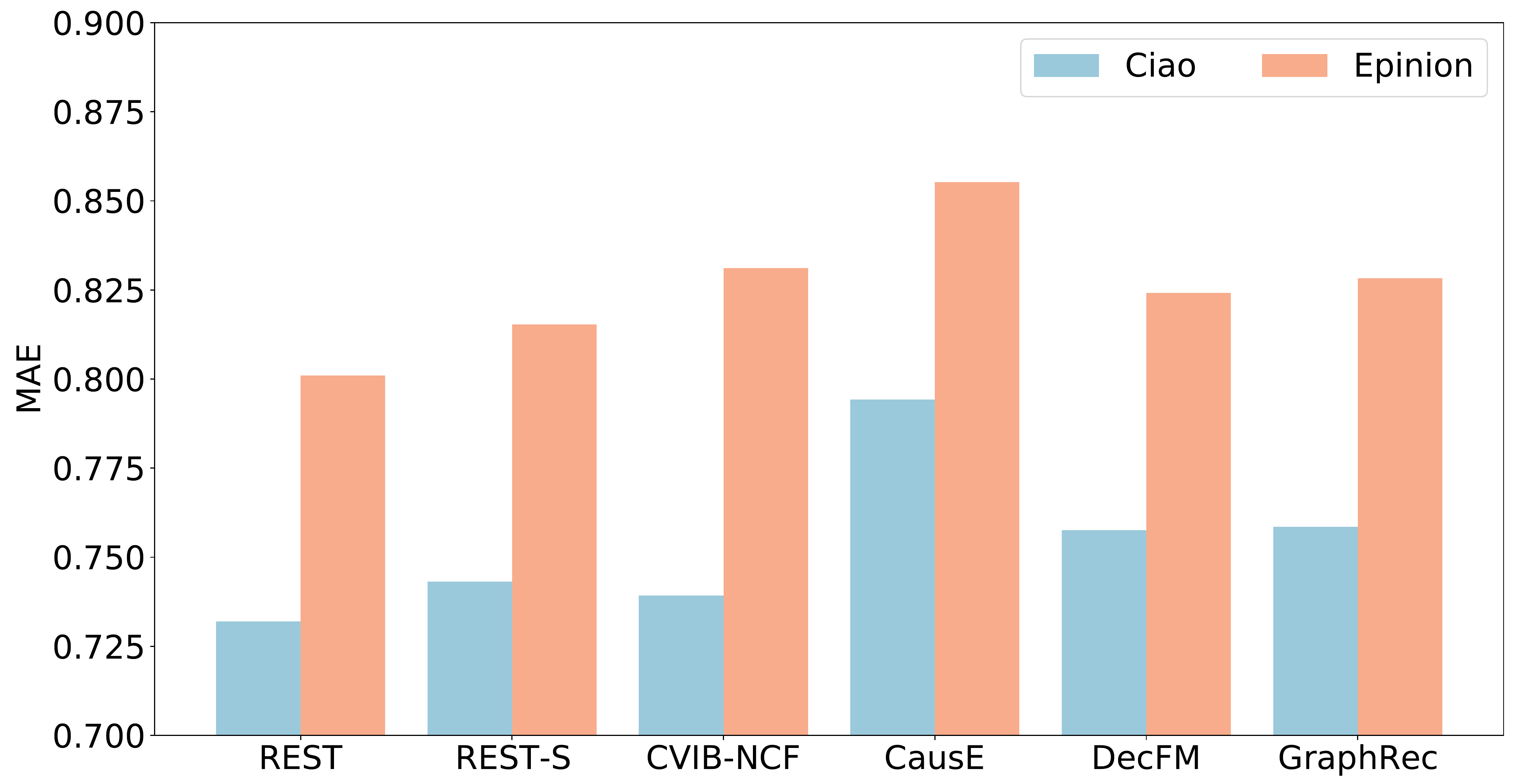}}
	\caption{MSE evaluation on Social networks and latent strategies variables. (\textit{best view in color})}
	\label{fig:mes_ablation}
\end{figure}

\begin{figure}[t]
\scalebox{0.477}
{\includegraphics[width=\textwidth]{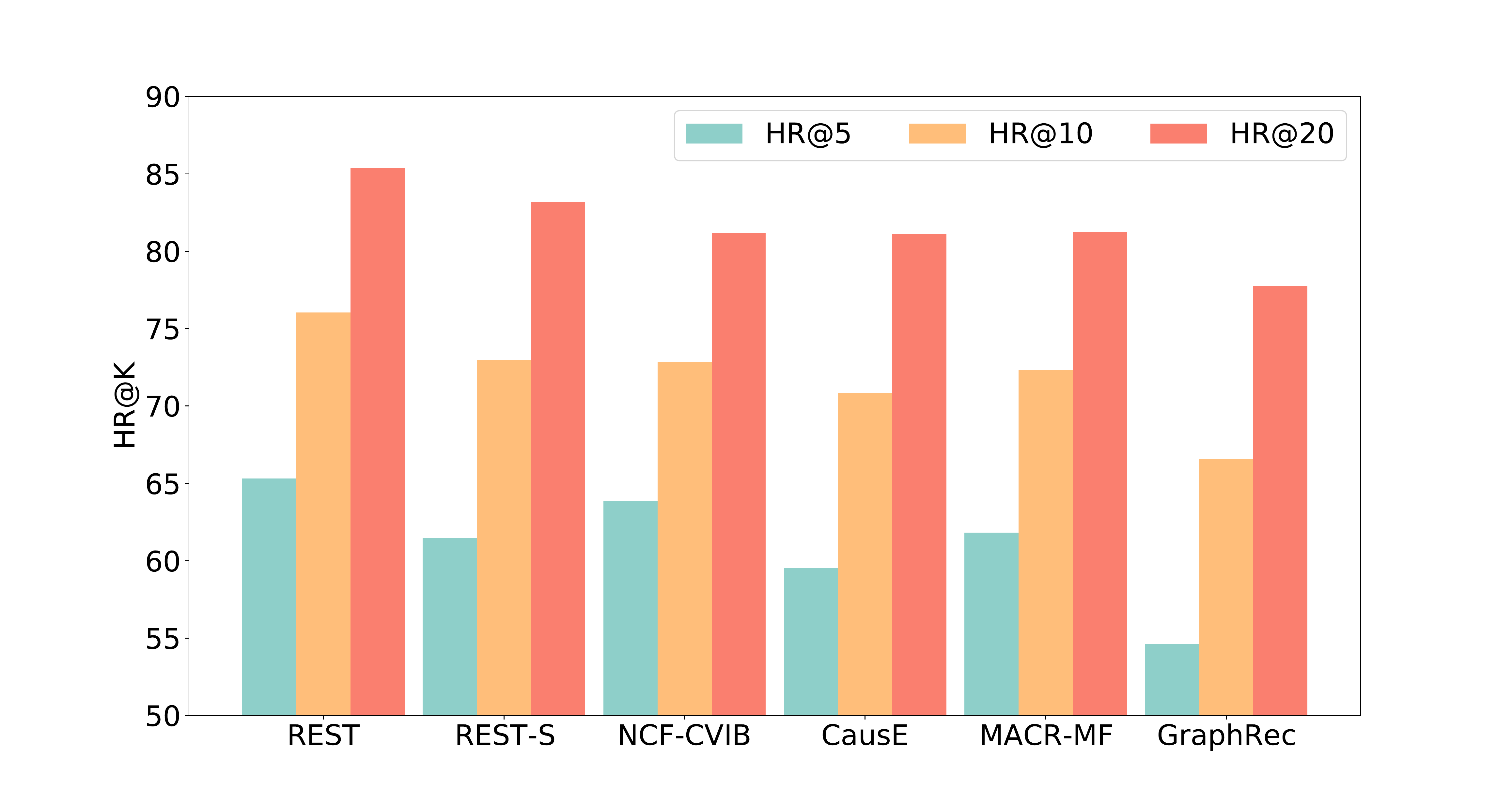}}
	\caption{HR@K evaluation on Social networks and latent strategies variables.(\textit{best view in color})}
	\label{fig:ndcg_ablation}
\end{figure}


\subsection{Visualization}
To further show the necessity for modeling the latent discrete strategies, we provide the visualization of the discrete strategies with 64 dimensions on the Ciao dataset, which is shown in Figure \ref{fig:visualization}. We split the 64 dimensions into 16 different categorical distributions which represent 16 different one-hot vectors. \textcolor{black}{Note that the horizontal axis stands for each dimension of latent variables. }We choose three different traditional festivals, Christmas, Thanksgiving Day, and Valentine's Day and draw the latent discrete strategies variables of the same user at each festival on different years. The value of the yellow block is $1$ and the value of the purple is $0$. According to the visualization, we can find that:
\textcolor{black}{
\begin{itemize}
    \item Shared patterns among the reconstructed strategy variables come from the same festival. For example, on Christmas, the locations of yellow blocks are similar. This means that the same festival shares similar promotion strategies.
    \item The latent discrete strategies variables from different festivals look different, which means that different festivals have different promotion strategies. By modeling the strategies variables, we can well model the complex user-item relationships despite the disadvantages of selection bias.
\end{itemize}
}
\begin{figure}[t]
\scalebox{0.49}
{\includegraphics[width=\textwidth]{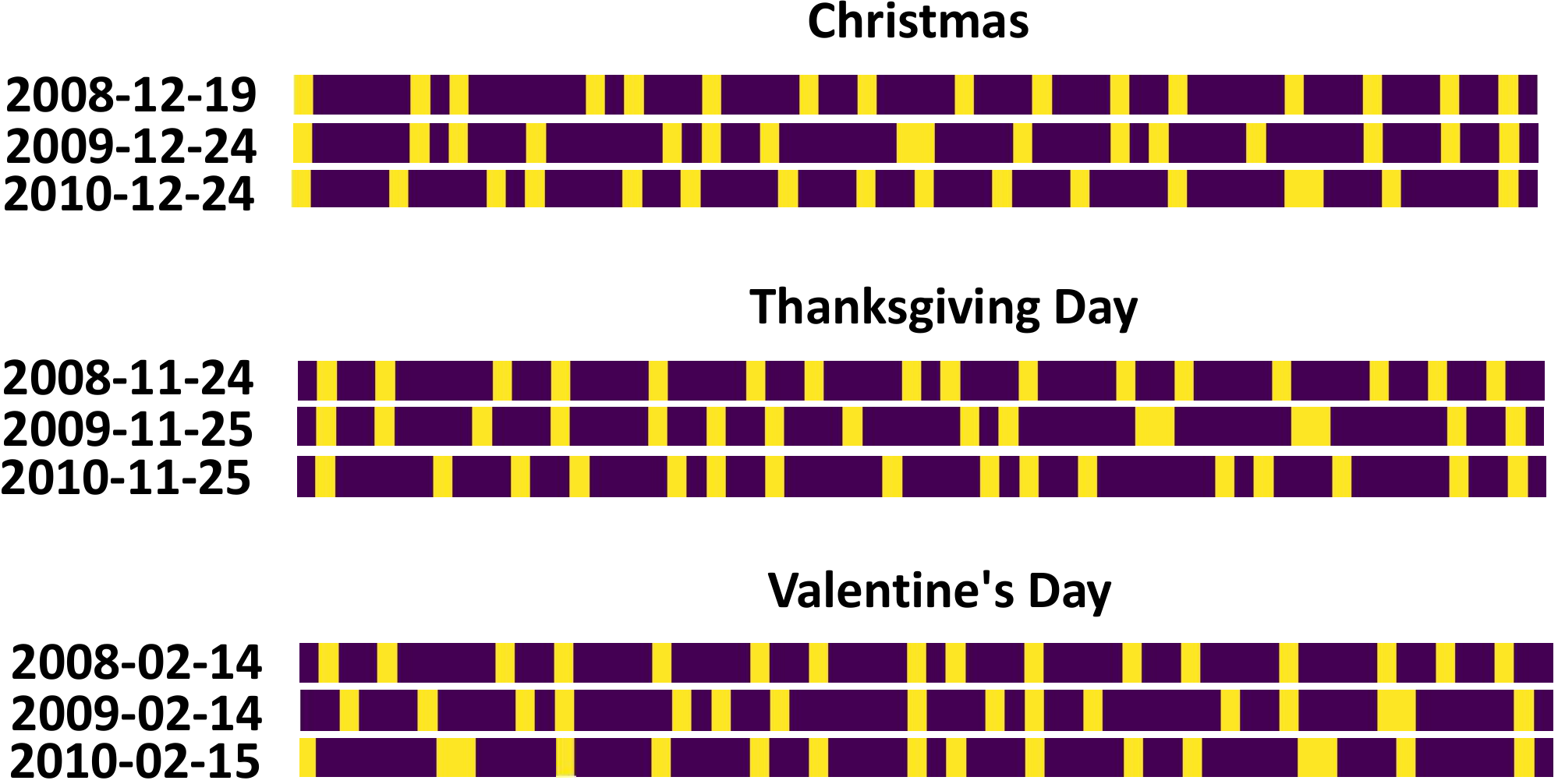}}
	\caption{The visualization of the discrete strategies variables. The vertical coordinate and the horizontal coordinate denote the different festivals and the different dimensions in the form of the one-hot vector. (\textit{best view in color})}
	\label{fig:visualization}
	\vspace{1.5em}
\end{figure}

\section{Conclusion}\label{sec:conclu}
This paper presents a debiased recommendation framework based on the explicitly modeling and reconstructing the discrete unobserved exposure strategies. In the proposed method, we reconstruct the latent exposure variables from the observational data using a variational auto-encoders framework, with the help of the clues from both the social networks and the items. The correctness, as well as the effectiveness of our proposal, are verified on four real-world datasets. The success of our model not only reveals that the latent exposure strategies are the cause of the well-known selection bias problem but also provides an effective solution for this open problem in the recommendation system. The visualization of the recovered exposure strategies on the real-world dataset also provides some interesting insights into the existing recommendation systems. 


\section{Acknowledgments}

We would like to thank Lingling Yi and Li Li from WeChat for their help and supports on this work.


%




\ifCLASSOPTIONcaptionsoff
  \newpage
\fi




\bibliographystyle{IEEEtran}
\bibliography{main_2}

\onecolumn
\appendix[]
The proofs of evidence lower bound (ELBO)
\label{supp:elbo}
\begin{equation}
\begin{split}
\ln  P\left(G, e, r, v\right)  \geq& P(v)-D_{KL}\left(Q(\bm{s}|e,r,v, u||P(\bm{s})\right) \\
- & D_{KL}\left(Q(u|G, e, r, v)||P(u)\right) \\
+ & \mathbb{E}_{Q(u|G, e, r, v)}\ln\left(P(G|u)\right)\\
+ & \mathbb{E}_{Q(u|G, e, r, v)}\mathbb{E}_{Q(\bm{s}|e,r,v, u)}\ln P(e|v, u, \bm{s})\\
+ & \mathbb{E}_{Q(u|G, e, r, v)}\ln P(r|u, v, e).
\end{split}
\nonumber
\end{equation}

\begin{proof}
	The proof of the ELBO is composed of three steps. First, we factorize the conditional distribution according to the Bayes theorem.
	\begin{equation*}
	\begin{aligned}
	&\ln P\left(G, e, r, v\right) = \ln \frac{P(G, e, r, v, u, \bm{s})}{P(u, \bm{s}|G, e, r, v)} \\
	=&\ln \frac{P(G, e, r, v, u, \bm{s})Q(u|G, e, r, v)Q(s|e, r, v, u)}{P(u|G, e, r, v)P(\bm{s}|e, r, v, u)Q(u|G, e, r, v)Q(\bm{s}|e, r, v, u)}
	\end{aligned}
	\end{equation*}
	Second, we add the expectation operator on both sides of the equation and reformalize the equation as follows:
	\begin{equation*}
	\begin{aligned}
	\ln P\left(G, e, r, v\right)=&D_{KL}\left(Q(u|G, e, r, v)||P(u|G, e, r, v)\right) +\\ & D_{KL}\left(Q(\bm{s}|e, r, v, u)||P(\bm{s}|e, r, v, u)\right) +\\
	& \mathbb{E}_{Q(u|G, e, r, v)}\mathbb{E}_{Q(\bm{s}|e, r, v, u)}\ln{\frac{P(G, e, r, v, u, \bm{s})}{Q(u|G, e, r, v)Q(\bm{s}|e, r, v, u)}}
	\end{aligned}
	\end{equation*}
	Third, we obtain the last equality with the help of $D_{KL}(\cdot||\cdot)\geq0$
	\begin{equation*}
	\begin{aligned}
	\ln P\left(G, e, r, v\right)\geq & \mathbb{E}_{Q(u|G, e, r, v)}\mathbb{E}_{Q(s|e, r, v, u)}\ln{\frac{P(G, e, r, v, u, \bm{s})}{Q(u|G, e, r, v)Q(\bm{s}|e, r, v, u)}}\\
	= & \mathbb{E}_{Q(u|G, e, r, v)}\mathbb{E}_{Q(\bm{s}|e, r, v, u)}\ln{\frac{P(u)P(G, e, r, v, \bm{s}|u)}{Q(u|G, e, r, v)Q(\bm{s}|e, r, v, u)}}\\
	= & \mathbb{E}_{Q(u|G, e, r, v)}\mathbb{E}_{Q(\bm{s}|e, r, v, u)}\ln{\frac{P(u)P(\bm{s}|u)P(G,r,e,v|u, \bm{s})}{Q(u|G, e, r, v)Q(\bm{s}|e, r, v, u)}}\\
	= & \mathbb{E}_{Q(u|G, e, r, v)}\mathbb{E}_{Q(s_{ij}|e, r_, v, u)}\ln{\frac{P(u)P(\bm{s})P(G|u)P(r,e,v|u, \bm{s})}{Q(u|G, e, r, v)Q(\bm{s}|e, r, v, u)}}\\
	= & \mathbb{E}_{Q(u|G, e, r, v)}\mathbb{E}_{Q(\bm{s}|e, r, v, u)}\ln{\frac{P(u)P(G|u)P(v)P(\bm{s})P(e|v,u,\bm{s})P(r|u, v,e,\bm{s}))}{Q(u|G, e, r, v)Q(\bm{s}|e, r, v, u)}}\\
	= & -D_{KL}\left(Q(u|G, e, r, v)||P(u)\right) - D_{KL}\left(Q(\bm{s}|e, r, v, u)||P(\bm{s})\right) + P(v) + \\ &\mathbb{E}_{Q(u|G, e, r, v)}\ln\left(P(G|u)\right)\\
	+ & \mathbb{E}_{Q(u|G, e, r, v)}\mathbb{E}_{Q(\bm{s}|e,r,v, u)}\ln P(e|v, u, \bm{s})\\
	+ & \mathbb{E}_{Q(u|G, e, r, v)}\ln P(r|u, v, e, \bm{s})
	\end{aligned}
	\end{equation*}
\end{proof}
\end{document}